\documentclass[lettersize,journal]{IEEEtran}
\usepackage{amsmath,amssymb,amsfonts,amsthm}
\usepackage{array}
\usepackage[caption=false,font=footnotesize,labelfont=rm,textfont=rm]{subfig}
\usepackage{textcomp}
\usepackage{stfloats}
\usepackage{url}
\usepackage{cases}

\newtheorem{theorem}{Theorem}

\usepackage{xcolor}
\usepackage{verbatim}
\usepackage{graphicx}

\usepackage{cite}
\usepackage{threeparttable}
\usepackage{booktabs}

\usepackage[ruled]{algorithm2e} 

\SetCommentSty{mycommfont}

\graphicspath{{./fig/}}

\def\BibTeX{{\rm B\kern-.05em{\sc i\kern-.025em b}\kern-.08em
    T\kern-.1667em\lower.7ex\hbox{E}\kern-.125emX}}
\usepackage{balance}

\begin{document}
\bstctlcite{setting}
\title{Digital Twin Assisted Deep Reinforcement Learning for Online Admission Control in Sliced Network}

\author{Zhenyu~Tao,
        Wei~Xu,~\IEEEmembership{Senior Member,~IEEE},
        Xiaohu~You*,~\IEEEmembership{Fellow,~IEEE}

\thanks{Z. Tao is with the National Mobile Communications Research Lab, Southeast University, Nanjing 210096, China (email: zhenyu\_tao@seu.edu.cn).}
\thanks{W. Xu and X. You are with the National Mobile Communications Research Lab, Southeast University, Nanjing 210096, China, and also with the Pervasive Communication Research Center, Purple Mountain Laboratories, Nanjing 211111, China (email: \{wxu, xhyu\}@seu.edu.cn).}
\thanks{X. You is the corresponding author of this paper.}

}

\markboth{IEEE TRANSACTIONS ON , ~Vol.~14, No.~8, August~2021}%
{Shell \MakeLowercase{\textit{et al.}}: A Sample Article Using IEEEtran.cls for IEEE Journals}

\maketitle

\begin{abstract}
The proliferation of diverse wireless services in 5G and beyond has led to the emergence of network slicing technologies. Among these, admission control plays a crucial role in achieving service-oriented optimization goals through the selective acceptance of service requests. Although deep reinforcement learning (DRL) forms the foundation in many admission control approaches thanks to its effectiveness and flexibility, initial instability with excessive convergence delay of DRL models hinders their deployment in real-world networks. We propose a digital twin (DT) assisted DRL solution to address this issue. Specifically, we first formulate the admission decision-making process as a semi-Markov decision process, which is subsequently simplified into an equivalent discrete-time Markov decision process to facilitate the implementation of DRL methods. A neural network-based DT is established with a customized output layer for queuing systems, trained through supervised learning, and then employed to assist the training phase of the DRL model. Extensive simulations show that the DT-assisted DRL improves resource utilization by over 40\% compared to the directly trained state-of-the-art dueling deep Q-learning model. This improvement is achieved while preserving the model's capability to optimize the long-term rewards of the admission process. 
\end{abstract}


%
\IEEEpeerreviewmaketitle

\section{Introduction}
\IEEEPARstart{I}{n} the past decades, the rapid development of communication technologies has led to the ever-growing expansion of network scale and the proliferation of diverse forms of network services, such as high-definition streaming videos, Internet of vehicles, and smart manufacturing facilities. As defined by the 3rd Generation Partnership Project (3GPP), 5G typical use cases include enhanced mobile broadband (eMBB), ultra-reliable low-latency communication (URLLC), and massive machine-type communications (mMTC), each with distinct quality-of-service (QoS) requirements \cite{3GPP.21.915,you2023toward}. 

To satisfy the varying demands of these heterogeneous services, network slicing has been introduced. Network slicing offers flexibility by managing tailored and logically isolated networks that share physical network resource. Multiple network slices coexist in a sliced physical network with limited total resource. Therefore, when conflicting or imminent conflicting service requests within different slices arrive, it is necessary to make choices among these requests to achieve specific objectives, such as maximizing long-term revenue for the infrastructure provider (InP) or realizing the fairness between different slices. This decision-making process is denoted as admission control. 

Conventional admission control approaches, such as searching methods or heuristic schemes, become ineffective or even fail to achieve the optimal solution due to the complexity of contemporary mobile networks \cite{van2019optimal}. Nevertheless, with the significant advancement of high-performance computing devices, researchers resort to learning-based methodologies\cite{xuwei2023-1}, \cite{xu2023toward}, particularly deep reinforcement learning (DRL). In DRL, deep neural networks are leveraged to handle systems with numerous states, and the rewards within DRL make it adaptable to various optimization targets.

While DRL-based admission control methods offer numerous advantages, challenges arise when deploying them in real networks. Early application of reinforcement learning (RL) can be traced back to games, exemplified by AlphaGo for the game of Go \cite{silver2016mastering} and as seen in OpenAI Five for the electronic game of Dota 2 \cite{berner2019dota}. These tasks share a crucial similarity: the training environment is exactly the same as the environment in which they are deployed, ensuring effective training and implementation. However, it is extremely challenging to create a precise virtual environment for network systems due to the complexity of the contemporary mobile network and the diversity of network services. In addition, an insufficiently accurate training environment can inevitably result in DRL models malfunctioning or ineffective when transferred to real networks. On the other hand, directly training DRL models on real networks, i.e., online optimization, always disrupts normal operations of the network and reduces system resource utilization due to the highly stochastic actions of randomly initialized DRL models. Therefore, there is an urgent need to investigate methodologies for the deployment of DRL models with as few as possible disruptions to real-world networks.

Considering the increasingly intricate structure and diversified functionality of modern wireless networks, digital twin (DT) technology, which enables the digitalization of specific physical entities, has been widely recognized as a viable solution for the development, deployment, and optimization of novel technologies with minimal interference to the real-world wireless networks. Currently, DTs for wireless networks have been established from various perspectives \cite{mihai2022digital}. In terms of the core network, programming \cite{rodrigo2023digital} and deep learning method \cite{tao2023deep} have been used to establish a digital twin for risk-free testbeds. Additionally, DT for radio access network (RAN) has been created through the collaboration among multiple carefully designed modules with various functionalities \cite{ren2023end}. For network topology, DT can be created based on graph neural networks \cite{wang2020graph}. In this paper, we employ DT to replicate the admission policy of an existing network and leverage the DT to assist in training the DRL model on the real network. This approach aims to mitigate the adverse impact associated with early-stage training, thus providing a DRL-based online admission control solution with reduced risk for the sliced wireless network. The main contributions of this paper are summarized as follows.

\begin{itemize}
\item{To the best of our knowledge, this is a first try that introduces DT into DRL models for addressing the instability issues within the initial training stage. It provides a practical solution to mitigate the initial stochasticity of the DRL model, thereby enhancing the deployment of DRL models in real-world networks.}
\item{We formulate the admission decision-making process within a network featuring request queues and combinatorial resources as a semi-Markov decision process. Subsequently, we transform it into a simplified but equivalent discrete-time Markov decision process to facilitate the implementation of DRL methods.}
\item{We introduce a neural network-based DT with a customized output layer for handling the queued requests, and leverage supervised learning to replicate network admission policies. An online optimization solution for admission control with DT-assisted DRL is developed, which exhibits significantly enhanced stability compared to traditional DRL training methods.}
\item{Extensive simulations are conducted to validate the effectiveness of the proposed solution. The results demonstrate that our approach significantly improves the resource utilization within the network, particularly during the initial training phase, while it also maintains the DRL model's performance in achieving specific service-oriented objectives.}
\end{itemize}

The remainder of this paper is organized as follows. Relevant works about the admission control for network slicing and the digital twin for mobile networks are reviewed in Section II. System model and problem formulation are described in Section III. Section IV elaborates on the proposed solution through two parts: the deep neural network-based DT and the DT-assisted DRL algorithm. Simulation results are presented and discussed in Section V. Finally, conclusions are drawn in Section VI.








\section{Related Work}

\subsection{Admission Control for Network Slicing}
Numerous studies have investigated admission control problems in sliced networks. Admission control for network slicing can be seen as an extension of call admission control \cite{jiang2022probabilistic}, where the admission policy of network services in different slices is designed to achieve specific targets like revenue maximization, priority assurance, and fairness guarantee. Distinct admission policies for incoming service requests from different slices result in various resource usage among slices. Consequently, the admission control for network slicing is also regarded as a resource allocation method with service requests as the finest granularity in certain literature like \cite{van2019optimal}. 

Conventional admission control mechanisms, e.g., first-come-first-served and random strategies, rely solely on the sequence of service requests and thus can hardly achieve designated goals. To realize the aforementioned targets, several approaches have been introduced. Jiang et al. \cite{jiang2016network} proposed an extensive searching method to improve user experiences within slices and increase network resource utilization. Soliman et al. \cite{soliman2016qos} designed a three-step heuristic scheme to achieve a trade-off between QoS and resource utilization. In \cite{dai2022psaccf}, a heuristic algorithm was proposed to amend priority violations and promote fairness. In \cite{bega2017optimising}, an adaptive algorithm was developed by applying Q-learning to maximize the InP revenue. While in \cite{haque20225g}, integer linear programming was adopted for admission control to maximize the revenue. However, it was pointed out in \cite{van2019optimal} that approaches like searching methods and heuristic algorithms may become inapplicable and cannot ensure optimality in complex network systems with a wide range of resource demands and services. Therefore, authors in \cite{van2019optimal} introduce DRL solutions into admission control tasks to maximize long-term revenue in network systems.

\subsection{Incorporation of DRL and Admission Control}
\label{DRLframework}
The standard DRL framework consists of an agent and an environment. The agent, guided by a policy, decides an action based on the environment’s state. The environment executes the action and provides the agent with a reward, through which the agent refines the policy. According to the implementation strategy of the agent, DRL methods are categorized into three groups. The first is the value-based (critic-only) methods, specifically, Deep Q-Network (DQN) \cite{mnih2013playing} and its variants such as double DQN \cite{van2016deep} and dueling DQN \cite{wang2016dueling}. These methods employ a deep neural network to rate each action by a Q-value, denoting its value for the current state. On the other hand, policy-based (actor-only) methods learn policy from cumulative rewards directly, such as REINFORCE \cite{williams1992simple} and G(PO)MDP \cite{baxter2001infinite}. Due to its high variance and large sampling costs, they are rarely employed in current DRL solutions. The last group is actor-critic methods, which combine value-based and policy-based methods. In such methods, the agent consists of an actor network to signify the probability of each action at current state, and a critic network to evaluate the action for the present state or directly assess the current state. Examples include Asynchronous Advantage Actor Critic (A3C) \cite{mnih2016asynchronous}, Proximal Policy Optimization (PPO) \cite{schulman2017proximal}, and Deep Deterministic Policy Gradient (DDPG) \cite{lillicrap2015continuous}.

Currently, considerable endeavors have been made to incorporate DRL into the admission control of sliced wireless networks. Villota-Jacome et al. \cite{villota2022admission} utilized DQN in the optimization of admission control policy, with the purpose of improving the service provider's profit and resource utilization. Troia et al. \cite{troia2022admission} performed both admission control and virtual network embedding based on advantage actor critic (A2C), the synchronous version of A3C. In \cite{sulaiman2022coordinated}, the authors adopted multi-agent PPO in both RAN slicing and admission control to improve long-term InP revenue.

Although DRL-based methods show outstanding performance in handling the admission control tasks, the instability of DRL at the initial training stage hinders the implementation of DRL in real network systems, especially for latency-critical services \cite{saha2023deep,liu2022deep,you20236g}.

\subsection{Digital Twin for Mobile Networks}
\label{dtsec}
DT is a key technology in creating digital replicas of complex systems, such as aviation, manufacturing, and architecture, to facilitate and evaluate virtual manipulations \cite{Jagatheesaperumal2023}. For mobile networks, DT enables the replication of real networks at different tiers through diverse techniques such as programming, mathematical modeling, and machine learning. Rodrigo et al. \cite{rodrigo2023digital} leveraged virtual machines to realize a DT of the 5G core network with two-way communication capability between real and virtual networks. In \cite{tao2023deep}, the authors adopted a deep learning method to construct a signaling-level DT of the control plane of a core network in a data-driven paradigm. Naeem et al.\cite{naeem2021digital} established a DT of network topology and utilized it in determining the optimal network slicing policy. In \cite{tang2023digital}, a DT of both network element and topology is realized for resource allocation in a sliced network. In \cite{wang2020graph}, a graph neural network-based DT was developed to mirror the network behavior and predict end-to-end latency. These methods leverage DT technology to replicate specific components of the wireless network, such as network topology, RAN, and core, enabling the evaluation and optimization of the physical network.

However, constructing a DT for the DRL environment of admission control tasks can hardly be achieved, due to the necessity not only to precisely digitize the entire wireless network but also to faithfully replicate the behavior of network service requests. Inadequate accuracy in the environment will result in a suboptimal or malfunctioning DRL model after being deployed in the real network. Therefore, in this study, we employ DT to replicate the admission policy of the real network, which is much more viable and implementable, and utilize it to enhance the training process of the DRL model for online admission control within the real sliced network. To the best of our knowledge, there have been few works that employed DT to address the instability issues encountered during the initial training phase of the DRL model.

\section{System Model and Problem Formulation}
\label{sec3}
We consider a typical network comprising three parties: end users, tenants, and InP. The InP is responsible for establishing separate logic network slices on the physical network infrastructure, which are tailored to satisfy tenant requirements. Tenants lease these slices from InP to serve the demands of their subscribers, namely end users. The services requested by end users are executed on slices provided by tenants, and charged based on the resource they utilize, including radio, computational power, and storage. We use $K$ to denote the number of slices, which corresponds to the number of tenants within the network. For a typical 5G network, for instance, we consider $K=4$ to represent a set of typical services including eMBB, URLLC, mMTC, and other.

\begin{figure}[!t]
\centering
\includegraphics[width=0.5\textwidth]{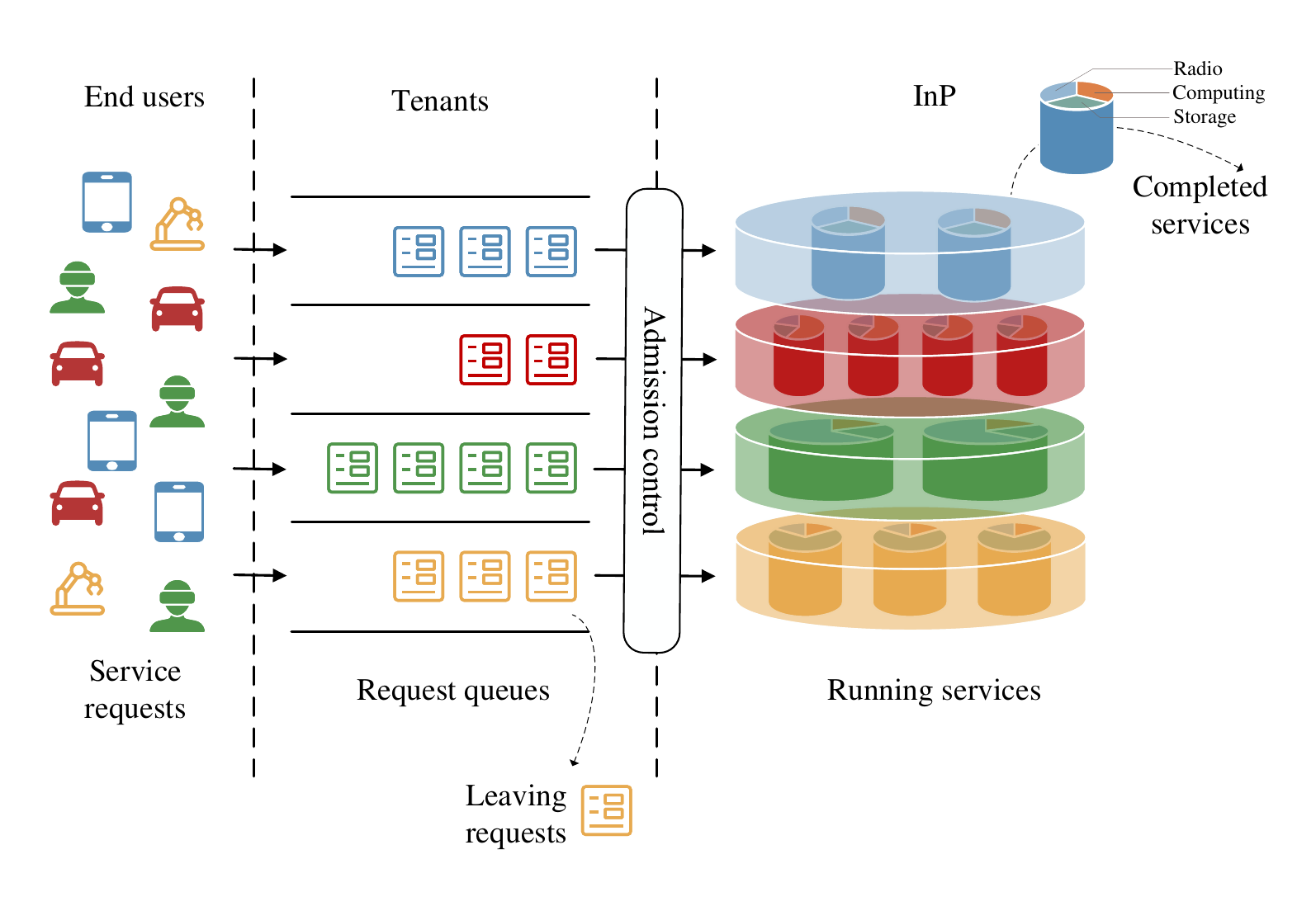}
\caption{Network slicing architecture with admission control}
\label{sys}
\end{figure}

Fig. \ref{sys} illustrated a system architecture of the above-mentioned network. Heterogenous service requests such as utilities, manufacturing, and online videos, are raised by end users. These requests are subsequently sent to tenants possessing the capability to provide relevant slices. With sufficient resource, the tenants then transfer requests to InP, thereby initiating the services on respective slices. However, it is common to encounter scenarios where excessive services are in operation, making remaining resources inadequate or inappropriate for accommodating particular services. In such instances, the corresponding service request should wait in queues for admission. The admission control policy is responsible for assessing both feasibility and priority of admitting requests. When a running service is complete, the occupied resource is released for reassignment. On the other hand, if a queued request experiences a waiting time that exceeds its patience threshold, it is going to be withdrawn from the queue.

Considering distinct application scenarios of these services, service requests and running services across different network slices will exhibit distinctions in terms of arrival patterns, service duration, waiting periods, etc.
Specifically, the arrival process of service requests with slice type $k \in \{1, 2,\dots, K\}$ follows the Poisson distribution with rate $\lambda_k$ and the service's resource occupation time, also known as service time, follows the exponential distribution with mean $1/\mu_k$. The maximum waiting time is set by a hold time $T_k$. If the waiting time of a request surpasses its hold time, it leaves the queue. Otherwise, if admitted, the service continues running until it reaches service time. In terms of resource, we define a vector $\mathbf{r}_k=\left[ r_k^R,r_k^C,r_k^S\right]$ to characterize the resource utilization of an active service, where $r_k^R,r_k^C$, and $r_k^S$ respectively represent the proportions of occupied radio, computing, and storage resources over the total resources. The occupied resource proportions across various services correspond to their distinct characteristics. For instance, services in eMBB slices utilize more radio resource to achieve broadband communication, while those in URLLC slices require slack computing resource to ensure low latency.

This wireless network system operates continuously and makes admission control decisions at any point in time. Thus, we adopt the semi-Markov decision process (SMDP) \cite{tijms2003first} to model the decision-making process in the network. Different from the discrete-time MDP where decisions are made at fixed time slots, the decision points in SMDP are triggered by events, with time intervals between events following a specific probability distribution. We denote SMDP by a 5-tuple $\left\langle \mathcal{S}, \mathcal{A}, \mathcal{T}, \mathcal{P}, \mathcal{R}\right\rangle$, where $\mathcal{S}$ and $\mathcal{A}$ respectively denote the state space and action space, $\mathcal{T}$ describes the distribution of the sojourn time, i.e., the duration between decision epochs, $\mathcal{P}$ represents the transition probability function, and $\mathcal{R}$ indicates the reward function. Functions $\mathcal{T}$, $\mathcal{P}$, and $\mathcal{R}$ possess the following Markovian property: if action $\mathbf{a}$ is chosen in state $\mathbf{s}$ at a decision epoch, then the sojourn time, the transition probability, and the reward depend only on the present state and the action chosen in this state until the next decision epoch.

\subsection{State Space}
The state $\mathbf{s}$ for each decision epoch can be defined by the number of requests $\mathbf{n}^\text{req}$ waiting in queues and the number of services $\mathbf{n}^\text{svc}$ running in the systems. Specifically, we define the state as
\begin{equation}
    \mathbf{s} = \left[\mathbf{n}^\text{req}, \mathbf{n}^\text{svc}\right],
\end{equation}
where 
\begin{align}
    \mathbf{n}^\text{req} &= \left[n_1^\text{req},\dots,n_k^\text{req},\dots,n_K^\text{req}\right],\\
    \mathbf{n}^\text{svc} &= \left[n_1^\text{svc},\dots,n_k^\text{svc},\dots,n_K^\text{svc}\right].
\end{align}

Resource constraints are introduced to ensure the occupied resource doees not exceed the accessible resource from the InP. Then, the state space $\mathcal{S}$ is formulated as follows
\begin{equation}
\begin{aligned}
    \mathcal{S} = \left\{\mathbf{s}=\left[\mathbf{n}^\text{req}, \mathbf{n}^\text{svc}\right]: \sum_{k=1}^K r_k^X n_k^\text{svc} \leq 1, X \in \{R, C, S\}  \right\} .
\end{aligned}
\end{equation}

\subsection{Action Space}
Considering the queue mechanism within the network, possible actions in this study are not simply binary, i.e., acceptance or rejection. Instead, the action is defined by a vector $\mathbf{a}$, specifying the number of admitted requests for each slice:
\begin{equation}
    \mathbf{a} = \left[n_1^\text{act},\dots,n_k^\text{act},\dots,n_K^\text{act}\right].
\end{equation}
where $n_k^\text{act} \in \left\{0,1,\dots,n_\text{max}\right\}$, and $n_\text{max}$ denotes the maximum number of simultaneously admitted requests within a single slice. When resource is sufficient, service requests are admitted immediately upon arrival, yielding actions in one-hot vectors, e.g., $\left[1, 0,\dots,0\right]$ and $\left[0, 1,\dots,0\right]$, which could never reach the limit $n_\text{max}$. The limit can only be reached in the following case. When faced with insufficient resources, the admission policy, denoted by $\pi$, accepts service requests selectively, leading to the accumulation of particular requests in queues. Under policy $\pi$, we define the maximum remaining resource when requests within slice $i$ begin to accumulate as follows
\begin{equation}
    \mathbf{r}_\text{re}\left(i\mid\pi\right)=\left[ r_\text{re}^R\left(i\mid\pi\right),r_\text{re}^C\left(i\mid\pi\right),r_\text{re}^S\left(i\mid\pi\right)\right].
\end{equation}
If an ongoing service with high resource utilization is complete, the admission policy may admit multiple requests of the same type in one decision epoch. The maximum number of requests admitted simultaneously in the same slice is defined as
\begin{equation}
    \begin{split}
        n_\text{max} = \max_{i,j} & \min_{X} \left(\Big{\lfloor}\frac{r_\text{re}^X\left(i\mid\pi\right) + r_{j}^X}{r_{i}^X}\Big{\rfloor}\right),
    \end{split}
\end{equation}
where $i,j\in\{1, 2,\dots, K\}$ refer to slice types, $X \in \{R, C, S\}$ represents the resource type, and $\lfloor.\rfloor$ denotes the floor function. The $n_\text{max}$ is related to the admission policy $\pi$ as well as the resource utilization characteristics $\mathbf{r}_k$ of services in slice. Now that we are ready to define the action space $\mathcal{A}$ as
\begin{equation}
\begin{aligned}
    \mathcal{A} = \{\mathbf{a} = [n_1^\text{act},\dots&,n_k^\text{act},\dots,n_K^\text{act}]:\\    
  0\leq &n_k^\text{act} \leq n_\text{max},\ \forall k \in \{1, 2,\dots, K\}\} .
\end{aligned}  
\end{equation}


\subsection{Sojourn Time Distribution}
The sojourn time represents the interval between adjacent decision epochs. Decisions are typically made when system state changes. In this network, the state $\mathbf{s}$ changes due to 3 events: request arrival, request departure, and service completion. When a request arrives, it is necessary to decide whether it should be admitted. Also, when a service is completed, the occupied resource is released and checked for service requests in the waiting queue. However, request departure does not necessitate a decision. With sufficient resource, there should be no queuing requests and thus no leaving requests. When the resource is inadequate or reserved for potential services with better rewards, the departure of queuing requests neither provides additional resource nor brings more new service requests. Therefore, only request arrival and service completion are considered to be the trigger events in our model.

For a queuing system, the sojourn time until the next trigger event depends on the arrival rate $\lambda_k$, the service rate $\mu_k$, and the number of ongoing services in each slice. Since the arrival process follows a Poisson distribution and the service process follows an exponential distribution, the sojourn time in state $\mathbf{s}$ follows the exponential distribution with an expectation of $\tau(\mathbf{s})$, defined as
\begin{equation}
    \tau(\mathbf{s}) = 1\left/\left(\sum_{k=1}^K \lambda_k + n_k^\text{svc} \mu_k\right)\right..
\end{equation}
That is, the arrival of the subsequent trigger event constitutes a Poisson process with the rate $1/\tau(\mathbf{s})$.

In SMDP, the decision $\mathbf{a}$ made at state $\mathbf{s}$ may change the number of ongoing services. This implies that the sojourn time depends not only on the state but also on the action in the current decision epoch. Moreover, only valid actions, adhering to constraints from request queues and resource capacities, can alter the number of ongoing services. We can define the valid action space at state $\mathbf{s}$ as
\begin{align}
\mathcal{A}_\text{va}(\mathbf{s}) = \bigg\{\mathbf{a} = [n_1^\text{act},\dots,n_k^\text{act},\dots,n_K^\text{act}]:\qquad\qquad\qquad&\notag\\    
  0\leq n_k^\text{act}\leq n_\text{max},\ n_k^\text{act} \leq n_k^\text{req},\ \forall k \in \{1, 2,\dots, &K\},\notag\\
  \sum_{k=1}^K r_k^X (n_k^\text{svc}+n_k^\text{act}) \leq 1,\ X \in \{R, C, S\}&\bigg\} .
\end{align}

Thus, $\tau(\mathbf{s})$ is more precisely written as
\begin{equation}
    \tau(\mathbf{s},\mathbf{a}) = 
    \begin{cases}        
        1\left/\left(\sum_{k=1}^K \lambda_k +n_k^\text{svc}\mu_k+n_k^\text{act}\mu_k\right)\right. ,&\mathbf{a}\in \mathcal{A}_\text{va}(\mathbf{s}),\\
        1\left/\left(\sum_{k=1}^K \lambda_k + n_k^\text{svc} \mu_k\right)\right. , &\text{otherwise.}
    \end{cases} 
\end{equation}

\subsection{Transition Probability}
The SMDP in this model includes an embedded Poisson process to describe the arrival process of trigger events, and an embedded discrete-time Markov chain to describe state transitions when an event occurs. The transition probability of the embedded Markov chain is defined by 
\begin{equation}
    p_{\mathbf{s},\mathbf{a},\mathbf{s}^\prime}=\text{Pr}(S_{t+1}=\mathbf{s}^\prime | S_{t}=\mathbf{s},A_t =\mathbf{a}),
\end{equation}
where $\text{Pr}(.)$ represent the probability function. $p_{\mathbf{s},\mathbf{a},\mathbf{s}^\prime}$ indicates the probability that if action $\mathbf{a}$ is chosen in the present state $\mathbf{s}$, the system will be in state $\mathbf{s}^\prime$ at the next decision epoch. By denoting the action chosen under policy $\pi$ in state $\mathbf{s}$ as $\mathbf{a}_{\mathbf{s}|\pi} \in \mathcal{A}$, we can rewrite the transition probability of the embedded Markov chain $\pi$ as $p_{\mathbf{s_0},\mathbf{a}_{\mathbf{s}_0|\pi},\mathbf{s}}$. Additionally, the equilibrium probability of the embedded Markov chain, given policy $\pi$ and state $\mathbf{s}$, is defined by
\begin{equation}
\label{equili}
    \omega(\mathbf{s}|\pi) =  \sum_{\mathbf{s}_0\in \mathcal{S}} \omega(\mathbf{s}_0|\pi) p_{\mathbf{s_0},\mathbf{a}_{\mathbf{s}_0|\pi},\mathbf{s}}.
\end{equation}

\subsection{Reward Function}
The reward function is defined to reflect not only positive effects of valid actions but also penalties of invalid actions. It is provisionally formulated as
\begin{equation}
    r(\mathbf{s},\mathbf{a})=
    \begin{cases}
        \text{Reward} , &\mathbf{a}\in \mathcal{A}_\text{va}(\mathbf{s}),\\
        \text{Penalty} ,&\text{otherwise.}
    \end{cases}
\end{equation}

Specifically, consider a system aimed at maximizing the InP revenue. Let $\mathbf{c}=\left[c^r,c^S,c^C\right]$ signify the per-unit charges of radio, computing, and storage resources per unit of time. Given a valid action $\mathbf{a}$ executed at state $\mathbf{s}$, the reward of total revenue accrued until the next trigger event is defined as
\begin{equation}
    \text{Reward} = \sum_{k=1}^{K} n_k^\text{act} \langle \mathbf{r}_k,\mathbf{c}\rangle \tau(\mathbf{s},\mathbf{a}),
\end{equation}
and the penalty reflecting the missed opportunities for resource optimization until the next trigger event is defined as
\begin{equation}
    \text{Penalty} = - \delta \tau(\mathbf{s},\mathbf{a}),
\end{equation}
where $c^R,c^S,c^C,$ and $\delta$ are all nonnegative constants. 

\subsection{Problem Formulation}
In order to formulate the optimization problem, it is necessary to prove that the long-term reward of SMDP is exclusively determined by the policy $\pi$. We use $R(t)$ to represent the total rewards up to time $t$. In the following theorem, we will prove that if the embedded Markov chain associated with policy $\pi$ has no disjoint closed sets, then the long-term average reward $g(\pi)$ for the SMDP is a constant and does not depend on the initial state $\mathbf{s}_0$.

\begin{theorem}
Suppose that the embedded Markov chain associated with policy $\pi$ has no disjoint closed sets. The long-term average reward for the SMDP
\begin{equation}
    \lim _{t \rightarrow \infty} \frac{R(t)}{t}=g(\pi), 
\end{equation}
for each initial state $\mathbf{s}_0$, where the constant $g(\pi)$ is given by 
\begin{equation}
\label{averagerevenue}
    g(\pi) = \frac{\sum_{\mathbf{s} \in \mathcal{S}} r\left(\mathbf{s},\mathbf{a}_{\mathbf{s}|\pi}\right) \omega(\mathbf{s}|\pi) }{ \sum_{\mathbf{s} \in \mathcal{S}} \tau\left(\mathbf{s},\mathbf{a}_{\mathbf{s}|\pi}\right) \omega(\mathbf{s}|\pi)},
\end{equation}
where $\omega(\mathbf{s}|\pi)$ refers to the equilibrium probability of the Markov chain given policy $\pi$ and state $\mathbf{s}$.
\end{theorem}
\begin{proof}
The proof of Theorem 1 is given in Appendix A.
\end{proof}

Now, the long-term average reward maximization problem is formulated as
\begin{equation}
\begin{split}
     \max _\pi \ &g(\pi)=\frac{\sum_{\mathbf{s} \in \mathcal{S}} r\left(\mathbf{s},\mathbf{a}_{\mathbf{s}|\pi}\right) \omega(\mathbf{s}|\pi) }{ \sum_{\mathbf{s} \in \mathcal{S}} \tau\left(\mathbf{s},\mathbf{a}_{\mathbf{s}|\pi}\right) \omega(\mathbf{s}|\pi)}\\
     \text { s.t. }\ &\sum_{\mathbf{s}\in \mathcal{S}} \omega(\mathbf{s}|\pi) = 1.
\end{split}  
\end{equation}

Due to the intricacy of the wireless network and diversified network services, the determination of equilibrium probability under policy $\pi$ is not straightforward. Therefore, we employ DRL to optimize policy and leverage neural networks to process extensive high-dimensional network data. Moreover, DT is utilized to assist in the training stage of the DRL model.

\section{DT-assisted Online DRL Solution}

\begin{figure*}[!t]
\centering
\includegraphics[width=1\textwidth]{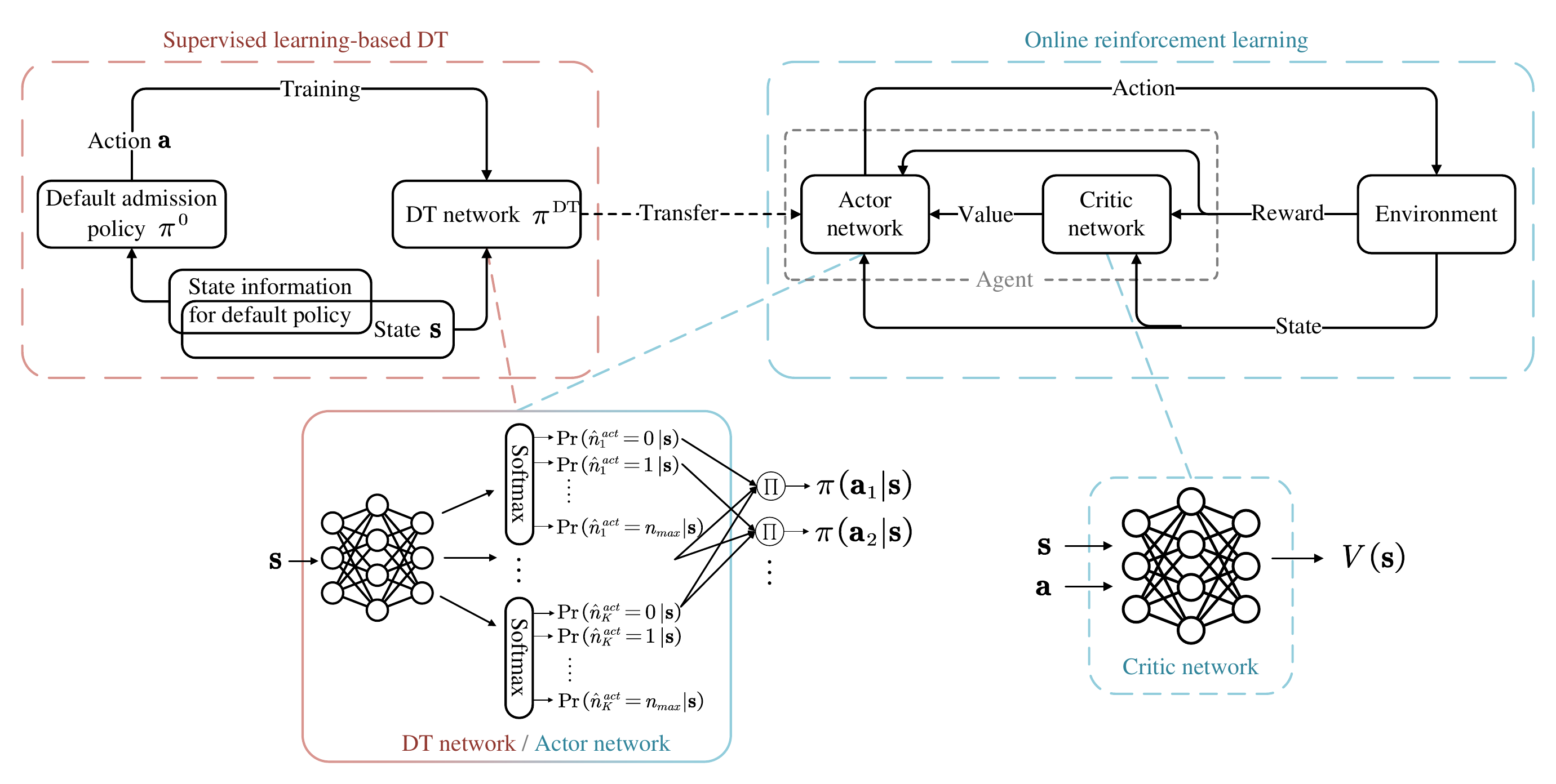}
\caption{Framework of DT-assisted online DRL solution}
\label{sche}
\end{figure*}

Due to the considerable stochasticity of the DRL model during its early training stage, the DRL methods can hardly be directly implemented in the real network. To settle this issue, we propose a DT-assisted online DRL solution in Fig. \ref{sche}. Our solution assumes that the real network has a default admission policy before employing DRL, even elementary ones such as the greedy policy. A neural network-based policy-level DT is established for this default admission policy through supervised learning. Once a network admission decision is made, training data, including the input data, e.g., network state $\mathbf{s}$, and the label, executed action $\mathbf{a}$, are collected in the form of a state-action pair $[\mathbf{s},\mathbf{a}]$. It is noteworthy that the network state information utilized by the default policy may not exactly be the same as the collected state $\mathbf{s}$. For example, the greedy policy depends only on the queuing requests $\mathbf{n}^\text{req}$ and the available resource, the latter of which is not within the defined network state $\mathbf{s}$. Nevertheless, these resources can be inferred from the ongoing services $\mathbf{n}^\text{svc}$ and the constant resource utilization vectors $\mathbf{r^{k}}$. We let the neural network learn such relations through training. During the establishment of DT, training data is collected by monitoring state and policy behavior in the real network, while the training process for DT is isolated from real networks, thus guaranteeing uninterrupted network operations.

After comprehensive learning, the DT network, capable of faithfully replicating the default admission policy, assists in training the DRL agent through transfer learning. Considering the similar functionalities of the DT network and the actor network in actor-critic DRL model, that is, parameterizing the policy through neural networks, we employ the actor-critic DRL model in this solution for a direct and efficient knowledge transfer. The neural network-based policy, essentially realizing the mapping from states to actions, is represented as $\pi(\mathbf{a}|\mathbf{s};\boldsymbol{\theta})$, signifying the probability of each action $\mathbf{a}$ in current state $\mathbf{s}$ with neural network parameter $\boldsymbol{\theta}$. Usually, the action with the highest probability is selected as the ultimate decision in the policy. The transfer learning enables the DRL agent to initially acquire a suboptimal default policy through the DT network and subsequently fine-tune towards the optimization goal using a pre-designed reward, thus mitigating the instability associated with training from scratch.

While the network state and action in the admission control task have been explicitly defined in the system model, the uncertain sojourn time in SMDP engenders highly variable rewards even with a fixed state and action. This results in confusion for the agent, hindering its ability to converge and optimize the policy. Fortunately, a transformation method \cite{tijms2003first} can be utilized to convert the SMDP into an equivalent discrete-time MDP such that for each stationary policy the long-term average reward in the discrete-time MDP is the same as that in the SMDP. The equivalent discrete-time MDP $\left\langle \Bar{\mathcal{S}}, \Bar{\mathcal{A}}, \Bar{\mathcal{P}}, \Bar{\mathcal{R}}\right\rangle$ is defined as
\begin{numcases}{}
    \Bar{\mathcal{S}} = \mathcal{S};\label{state} \\
    \Bar{\mathcal{A}} = \mathcal{A};\\
    \Bar{r}(\mathbf{s},\mathbf{a}) = r(\mathbf{s},\mathbf{a})/\tau(\mathbf{s},\mathbf{a}), \quad\quad\ \, \mathbf{a}\in \mathcal{\Bar{\mathcal{A}}}\text{ and }\mathbf{s} \in \Bar{\mathcal{S}};\label{rewardbar}\\
    \Bar{p}_{\mathbf{s},\mathbf{a},\mathbf{s}^\prime} = 
    \begin{cases}
        (\tau_0/\tau(\mathbf{s},\mathbf{a}))p_{\mathbf{s},\mathbf{a},\mathbf{s}^\prime}, \\ 
        \quad\quad\quad\quad\mathbf{s}\neq \mathbf{s^\prime}, \mathbf{a}\in \mathcal{\Bar{\mathcal{A}}} \text{ and } \mathbf{s}, \mathbf{s^\prime} \in \Bar{\mathcal{S}};\\
        (\tau_0/\tau(\mathbf{s},\mathbf{a}))p_{\mathbf{s},\mathbf{a},\mathbf{s}^\prime}+\left(1- (\tau/\tau(\mathbf{s},\mathbf{a}))\right),\\ 
        \quad\quad\quad\quad\mathbf{s} = \mathbf{s^\prime}, \mathbf{a}\in \mathcal{\Bar{\mathcal{A}}} \text{ and } \mathbf{s}, \mathbf{s^\prime} \in \Bar{\mathcal{S}}; 
    \end{cases} \label{transitionp}     
\end{numcases}
where $\tau_0$ is a constant with $0<\tau_0\leq \min _{\mathbf{s},\mathbf{a}} \tau(\mathbf{s},\mathbf{a})$. In the following theorem, we prove the equivalence, in terms of long-term average reward, between this discrete-time MDP and the original MDP.

\begin{theorem}
Given the embedded Markov chain associated with policy $\pi$ in SMDP has no disjoint closed sets, we have:
\begin{equation}
    g(\pi)=\bar{g}(\pi),
\end{equation}
where $g(\pi)$ and $\bar{g}(\pi)$ is the long-term average reward for SMDP and its equivalent discrete MDP.
\end{theorem}
\begin{proof}
The proof is given in Appendix B.
\end{proof}

Note that the embedded Markov chain in the SMDP of the system model is a unichain for all stationary policies $\pi$, satisfying the equivalence requirement in \textbf{Theorem 2}. Finally, the DRL agent is able to refine the admission control policy through the equivalent reward function expressed as
\begin{equation}\label{rewardbar_1}
    \Bar{r}(\mathbf{s},\mathbf{a}) = r(\mathbf{s},\mathbf{a})/\tau(\mathbf{s},\mathbf{a}) =     \begin{cases}
        \sum_{k=1}^{K} n_k^\text{act} \langle \mathbf{r}_k,\mathbf{c}\rangle, &\mathbf{a}\in \mathcal{A}_\text{va}(\mathbf{s}),\\
        - \delta,&\text{otherwise.}
    \end{cases}
\end{equation}

In terms of the model structure, the DT network mainly consists of a standard multilayer feed-forward network (FFN), with targeted modifications to the output layer for this task. In conventional approaches, the output layer generates values representing all actions in the action space $\mathcal{A}$, which are subsequently transformed into predicted probabilities via a softmax activation function. This structure proves concise and effective for systems without request queues, where the action space includes solely acceptance and rejection options. However, when dealing with tasks involving request queues, the number of potential actions escalates to $(n_\text{max}+1)^K$, posing challenges in training an effective network. In addition, the conventional structure ignores the inherent relationships among the predicted probabilities of different values for a single variable within the action vector, that is,
\begin{equation}
    \sum_{n=0}^{n_\text{max}} \text{Pr}(\hat{n}_k^\text{act} = n|\mathbf{s}) = 1,\quad \forall k \in \{1, 2,\dots, K\},
\end{equation}
where $\hat{n}_k^\text{act}$ denotes the $k$-th predicted value in the action vector. To settle this problem, we have the output layer separately compute the predicted probability of different values for each variable, rather than for each action. In this new structure, the probability of action $\mathbf{a}$ is derived from the product of the probability for each variable, as expressed by:
\begin{equation}
\pi^\text{DT}(\mathbf{a}|\mathbf{s}) = \prod_{k=1}^K \text{Pr}(\hat{n}_k^\text{act} = {n}_k^\text{act}|\mathbf{s}),\quad\mathbf{a} = \left[n_1^\text{act},\dots,n_K^\text{act}\right],
\end{equation}
which is used in backpropagation and parameter updating in the training phase. During the prediction phase, variables are determined through a greedy algorithm or probability-based sampling, then concatenated to construct the predicted action vector $\mathbf{\hat{a}}$. This modification reduces the number of nodes in the output layer from $(n_\text{max}+1)^K$ to $K(n_\text{max}+1)$, substantially alleviating the training challenges of prohibitively high computational complexity.

Due to the consistent functionality between the DT network and the actor network, we directly integrate the DT network structure as the actor network to parameterize the agent policy by $\pi(\mathbf{a}|\mathbf{s};\boldsymbol{\theta})$. The critic network is established following the conventional form in the A2C algorithm. To introduce the A2C algorithm, we begin by defining the state and action value functions:
\begin{align}
V(\mathbf{s}) & =\mathbb{E}\left[\sum_{t=0}^{\infty} \gamma^t \bar{r}\left({\mathbf{s}_t},\mathbf{a}_t\right) \bigg| \mathbf{s_0}=\mathbf{s}\right], \\
Q(\mathbf{s}, \mathbf{a}) & =\mathbb{E}\left[\sum_{t=0}^{\infty} \gamma^t \bar{r}\left({\mathbf{s}_t},\mathbf{a}_t\right) \bigg| \mathbf{s_0}=\mathbf{s}, \mathbf{a_0}=\mathbf{a}\right],
\end{align}
where $\gamma$ is the discount factor that represents how far future rewards are taken into account at this moment. The state value function describes the cumulative rewards initiated from the current state $\mathbf{s}$, while the action value function additionally considers the impact of the current action $\mathbf{a}$ on the cumulative rewards. When the next state is identified as $\mathbf{s}^\prime$, we can rewrite the action value function using the one-step reward and the state value function as follows
\begin{equation}
\begin{split}
    Q(\mathbf{s}, \mathbf{a}) &= \bar{r}\left(\mathbf{s},\mathbf{a}\right) + \mathbb{E}\left[\sum_{t=1}^{\infty} \gamma^t \bar{r}\left({\mathbf{s}_t},\mathbf{a}_t\right) \bigg| \mathbf{s_1}=\mathbf{s}^\prime\right]\\   
    & = \bar{r}\left(\mathbf{s},\mathbf{a}\right) + \gamma V(\mathbf{s}^\prime),
\end{split}
\end{equation}
The advantage function, indicating the degree to which the action $\mathbf{a}$ performs better or worse than the average action in state $\mathbf{s}$, is defined as
\begin{align}
    A(\mathbf{s},\mathbf{a}) &= Q(\mathbf{s}, \mathbf{a}) - V(\mathbf{s})\notag\\
    &= \bar{r}\left(\mathbf{s},\mathbf{a}\right) + \gamma V(\mathbf{s}^\prime) - V(\mathbf{s}).
\end{align}
As a result, we can use a single critic network to parameterize the state value function and calculate the advantage function for the current action. The critic network consists of a multilayer FFN and a one-node output layer with inherent parameters $\boldsymbol{\theta}_V$, while the parameterized state value function is denoted as $V(\mathbf{s};\boldsymbol{\theta}_V)$.

The training strategy of our proposed solution can be divided into two stages. The first is the supervised learning stage for the DT network. We consider the admission control as a classification task, where $\mathbf{s}$ and $\mathbf{a}$ in the dataset serve as the input and label respectively. The cross-entropy loss is employed to train the DT network for replicating the default network policy, which is mathematically defined as
\begin{equation}
\label{celoss}
\mathcal{L}_{\boldsymbol{\theta}_\text{DT}}=-  \log \pi^\text{DT}(\mathbf{a}|\mathbf{s};\boldsymbol{\theta}_\text{DT}).
\end{equation}
The training process of the DT network is described in \textbf{Algorithm \ref{al1}}. 

\begin{algorithm}[t]
	\caption{Supervised learning for DT network}\label{al1}
        Collect network state $\mathbf{s}$ and default policy action $\mathbf{a}$ each time the network makes an admission decision.
        
        Construct datasets for training and validation. 
        
        Initialize the DT network with random parameters $\boldsymbol{\theta}_\text{DT}$.
        
	\For{$episode \leftarrow 1$ to $T$}{
        Calculate cross-entropy loss $\mathcal{L}_{\boldsymbol{\theta}_\text{DT}}$ according to (\ref{celoss}) on training dataset.
        
        Update $\boldsymbol{\theta}_\text{DT}$ via gradient descent on $\mathcal{L}_{\boldsymbol{\theta}_\text{DT}}$.
        
        Check the average cross-entropy loss and predictive accuracy of the DT network on validation dataset.
	}
\end{algorithm}

The following stage is DT-assisted online DRL training. The state value function satisfies the Bellman equation and can be recursively defined as
\begin{equation}\label{vfunc}
    V(\mathbf{s}) =\mathbb{E}\left[ \bar{r}\left(\mathbf{s},\mathbf{a}\right) + \gamma V(\mathbf{s}^\prime)\right].
\end{equation}
Therefore, the loss function for the critic network with the parameter $\boldsymbol{\theta}_V$ takes the following form:
\begin{equation}\label{loss1}
    \mathcal{L}_{\boldsymbol{\theta}_V} =\left(\bar{r}\left(\mathbf{s},\mathbf{a}\right) + \gamma V(\mathbf{s}^\prime;\boldsymbol{\theta}_V) -V(\mathbf{s};\boldsymbol{\theta}_V)\right)^2.
\end{equation} 
Meanwhile, the loss function for the actor network is defined as
\begin{equation}\label{loss2}
    \mathcal{L}_{\boldsymbol{\theta}} = \log \pi(\mathbf{a}|{\mathbf{s}} ; \boldsymbol{\theta})A(\mathbf{s},\mathbf{a}),
\end{equation}
in order to optimize the policy by favoring actions with higher advantages, and thereby maximize long-term rewards.

The training of both networks is realized through continuous interaction with the real network. Specifically, given the current state $\mathbf{s}$, the actor network makes an action decision $\mathbf{a}$ under its policy $\pi(\mathbf{a}|{\mathbf{s}})$. The network implements this chosen action, providing feedback in the form of reward $\bar{r}(\mathbf{s},\mathbf{a})$ and the next state $\mathbf{s}^\prime$. Variables $\mathbf{s}$, $\bar{r}(\mathbf{s},\mathbf{a})$, and $\mathbf{s}^\prime$ are used to calculate the loss functions defined in (\ref{loss1}) and (\ref{loss2}), adjusting parameters via gradient descent. 

In order to stabilize the DRL model, we perform the initialization with $\boldsymbol{\theta} = \boldsymbol{\theta}_\text{DT}$ to transfer the parameter in the DT network to the actor network before training. However, the parameters within the critic network are randomly initialized and will disrupt the actor network. To settle this issue, we adopt a two-step training approach to prevent the stable policy from returning stochastic. Firstly, we freeze the actor network and individually train the critic network. In case the DT network faithfully replicates the default admission policy, the training of the critic network does not disrupt the normal operation of the real network, as the policy within the actor network remains unchanged. This training stage persists until the critic network achieves a relatively accurate approximation of the state value function $V(\mathbf{s};\boldsymbol{\theta}_V)$. After that, we unfreeze the actor network and simultaneously train both networks, maximizing long-term rewards through the fine-tuning of $\boldsymbol{\theta}$ and $\boldsymbol{\theta}_V$. A detailed description of this process is provided in \textbf{Algorithm \ref{al2}}.

\begin{algorithm}[t]
	\caption{DT-assisted Online DRL solution}\label{al2}

        \tcp{ Step 1: Train DT network}
            
        Implement the DT network through \textbf{Algorithm \ref{al1}}

	\tcp{Step 2: Train critic network} 
        Initialize actor network as a copy of DT network with parameter $\boldsymbol{\theta}=\boldsymbol{\theta}_\text{DT}$
        
        Initialze critic network with random parameter $\boldsymbol{\theta}_V$

        \For{$episode \leftarrow 1$ to $T$}{
            Get state $\mathbf{s}$ from environment.
            
            Perform action $\mathbf{a}$ according to policy $\pi(\mathbf{a}|\mathbf{s} ; \boldsymbol{\theta})$.

            Get the next state $\mathbf{s}^\prime$, reward $r(\mathbf{s},\mathbf{a})$ from environment.

            Assess state values $V(\mathbf{s};\boldsymbol{\theta}_V)$ and $V(\mathbf{s}^\prime;\boldsymbol{\theta}_V)$ through critic network.

            Calculated $\mathcal{L}_{\boldsymbol{\theta}_V}$  according to (\ref{loss1}) and update ciritc network parameter $\boldsymbol{\theta}_V$ by performing gradient descent.        
        }

        \tcp{Step 3: Train both actor and critic networks} 
        
        \For{$episode \leftarrow 1$ to $T$}{
            Get state $\mathbf{s}$ from environment.
            
            Perform action $\mathbf{a}$ according to policy $\pi(\mathbf{a}|\mathbf{s} ; \boldsymbol{\theta})$.

            Get the next state $\mathbf{s}^\prime$, reward $r(\mathbf{s},\mathbf{a})$ from environment.

            Assess state values $V(\mathbf{s};\boldsymbol{\theta}_V)$ and $V(\mathbf{s}^\prime;\boldsymbol{\theta}_V)$ through critic network.
            
            Calculated $\mathcal{L}_{\boldsymbol{\theta}}$  according to (\ref{loss2}) and update actor network parameter $\boldsymbol{\theta} $ via gradient descent.
             
            Calculated $\mathcal{L}_{\boldsymbol{\theta}_V}$  according to (\ref{loss1}) and update ciritc network parameter $\boldsymbol{\theta}_V $ via gradient descent.  
        }
\end{algorithm}

\section{Experimental Evaluation}

\subsection{Experiment Setting}
The simulation of the network system, DT network, and DRL model in this study are implemented based on Python 3.9, Pytorch 1.10, CUDA 11.3, and Numpy. The experimentation is performed on a commercial PC (i7-12700KF CPU, Windows 11 64-bit operating system, and 32 GB RAM) with a dedicated GPU (NVIDIA GeForce RTX 3080). 

\renewcommand{\arraystretch}{1.2}
\begin{table}[]
    \centering
    \caption{Environment settings}
    \begin{tabular}{cccc}
        \toprule \textbf{Symbol} & \textbf{Value}& \textbf{Symbol} & \textbf{Value}\\
        \midrule 
        $K$ & 4 & $n_\text{max} $ & 3\\
        $\lambda_1 $ &  4 &  $\lambda_2 $  &  3.6  \\
        $\lambda_3 $ &  3.2 &  $\lambda_4 $  &  2.8  \\
        $1/\mu_1 $ &  3.2 &  $1/\mu_2 $  &  4  \\
        $1/\mu_3 $ &  1.6 &  $1/\mu_4 $  &  2.4  \\
        $T_1 $ &  0.8 &  $T_2 $  &  1  \\
        $T_3 $ &  0.2 &  $T_4 $  &  0.6  \\
        $\mathbf{r_1}$ &  $[0.02,0.03,0.04]$ &  $\mathbf{r_2}$  &  $[0.04,0.02,0.016]$   \\
        $\mathbf{r_3}$ &  $[0.016,0.04,0.016]$ &  $\mathbf{r_4}$ &  $[0.024,0.024,0.024]$   \\

        \bottomrule
    \end{tabular}
    \label{tab0}
\end{table}

The parameter setting of the network environment is outlined in Table \ref{tab0}. As previously discussed in Section \ref{sec3}, the slices encompass mMTC, eMBB, URLLC, and other, which correspond to 1, 2, 3, and 4 in the table. Parameters for each slice are determined based on their respective features. For example, the URLLC service shows the shortest mean service time $1/\mu_3 = 1.6$ and hold time $T_3 = 0.2$, as well as the maximum computing resource utilization $r_3^C=0.04$. In contrast, the services in mMTC and eMBB slices exhibit the highest utilization of storage resource $r_1^S=0.04$ and radio resource $r_2^R=0.04$ respectively. 

\begin{table}[]
    \centering
    \caption{Training settings}
    \begin{tabular}{cc}
        \toprule \textbf{Symbol} & \textbf{Value}\\
        \midrule 
        Dimension of models & 64 \\
        Number of layers  & 3 \\
        Batch size for DT & 64 \\
        Learning rate for DT& 1e-4\\
        Learning rate for critic & 1e-4\\
        Learning rate for actor & 4e-4\\
        $\gamma$ in calculation of $A(\mathbf{s},\mathbf{a})$ &0.99\\       
        \bottomrule
    \end{tabular}
    \label{tab1}
\end{table}

In terms of the models, we choose the FFN with 3 layers and 64 nodes within each layer. The dimension of FFN is identical in all three networks including DT network, actor network, and critic network. In the supervised learning phase for the DT network, we construct a dataset with collected data and then train the network, thus we can employ batch training with a batch size of 64 to reduce the fluctuations. On the contrary, during the training of actor and critic networks, only one set of data can be obtained per decision epoch, so we use a batch size of 1 in this scenario. The additional training configurations can be found in Table \ref{tab1}. The pre-trained actor network necessitates a relatively higher learning rate to deviate from the original policy, therefore the learning rate for the actor network exceeds that for the critic network in the configuration. 

Three distinct default admission policies are chosen in our experiment to comprehensively evaluate the performance of our solution. The first policy employs a heuristic algorithm considering priority (defined as URLLC $>$ eMBB $>$ mMTC $>$ other in our experiment) and fairness among different slices, as detailed in \cite{dai2022psaccf}. We shall abbreviate this policy as PRIO throughout the remainder of this paper. The second one uses integer linear programming (ILP) to maximize the radio resource utilization at each decision epoch \cite{haque20225g}. The third one employs a straightforward greedy algorithm that accepts requests based on the decreasing order of radio resource occupation.

Furthermore, we employ the state-of-the-art Dueling-DQN method for comparative analysis against our proposed DRL approach. The Dueling-DQN model is configured with a similar architecture comprising three layers, each containing 64 nodes. In the output layer, we retain its conventional structure, aligning the number of nodes with the count of potential actions, calculated as $(n_\text{max}+1)^K = 256$, as opposed to the proposed modified output layer structure.

To fully demonstrate the effectiveness of our proposed approach, we select a different optimization goal - maximizing revenue from storage resource charges. The reward is calculated by (\ref{rewardbar_1}), with the charge vector $\mathbf{c}=[0,0,100]$.

\subsection{Evaluation Results}
\subsubsection{Supervised Learning-based DT Performance Evaluation}
\begin{figure}[!t]
\centering
\includegraphics[width=0.5\textwidth]{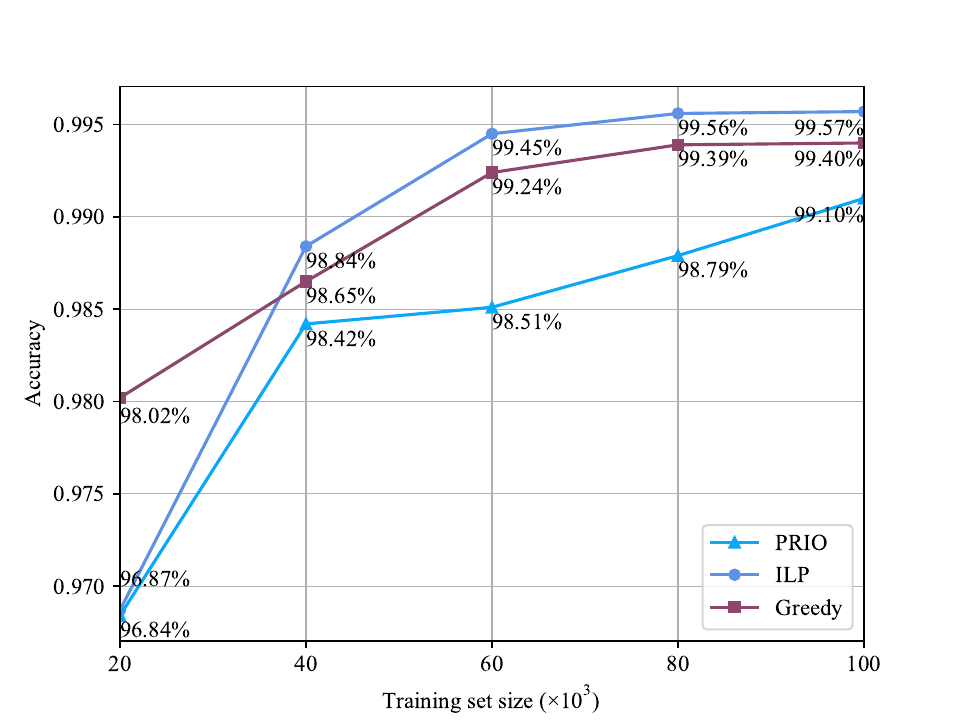}
\caption{Predicive accuracy of DT network with different default policies}
\label{DTacc}
\end{figure}
We configured the training epochs for our DT network as 400. To prevent overfitting, we employed the early-stopping technique with a patience of 20 epochs. Figure \ref{DTacc} illustrates the predictive accuracy of the DT network on the test set under different default admission policies and varying training set sizes. The results reveal a positive correlation between predictive accuracy and training set size, with accuracy stabilizing as the training sample size increases. Notably, when the training set size reaches 100,000 samples, the predictive accuracy of the DT network exceeds 99\% for all three policies, indicating a faithful replication of the default admission policies. As discussed in the last section, the process of collecting training samples does not disrupt the normal operation of network systems. Consequently, we employ the DT network trained on a 100,000-sample dataset for subsequent experiments.



\subsubsection{DT-assisted DRL Performance Evaluation}
The performance of an admission policy can be analyzed across three dimensions: cumulative rewards, resource utilization, and the acceptance ratio of requests within different slices \cite{van2019optimal,villota2022admission,bega2019machine}. In this study, during the training phase, we compare resource utilization and acceptance ratio among different methods to assess their impact on the network system. After training completion, cumulative rewards are used to check whether the optimization goal has been achieved.

\begin{figure}[!t]
\centering
\subfloat[PRIO]{\includegraphics[width=0.248\textwidth]{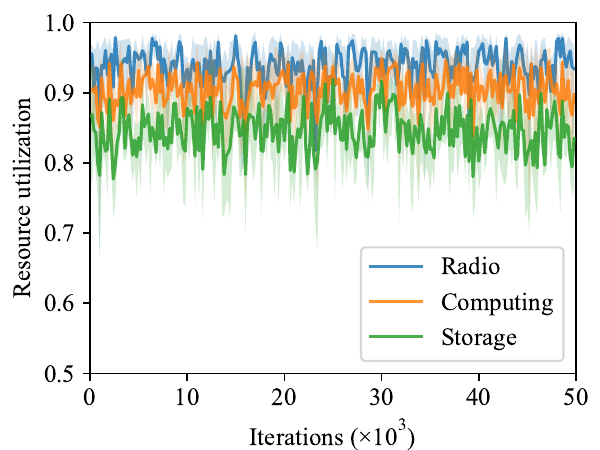}%

}
\subfloat[PRIO]{\includegraphics[width=0.248\textwidth]{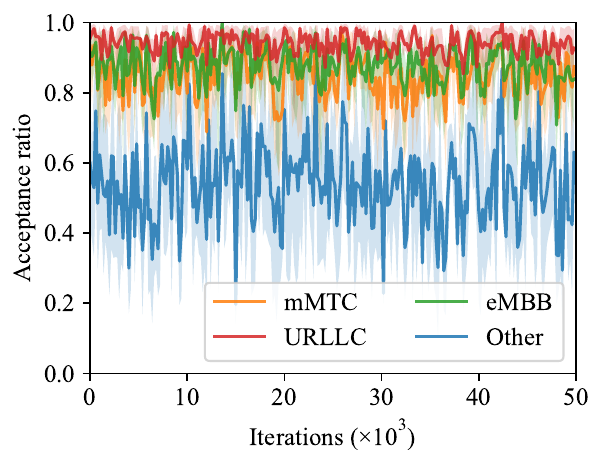}%
\label{PRIOAR}}

\subfloat[ILP]{\includegraphics[width=0.248\textwidth]{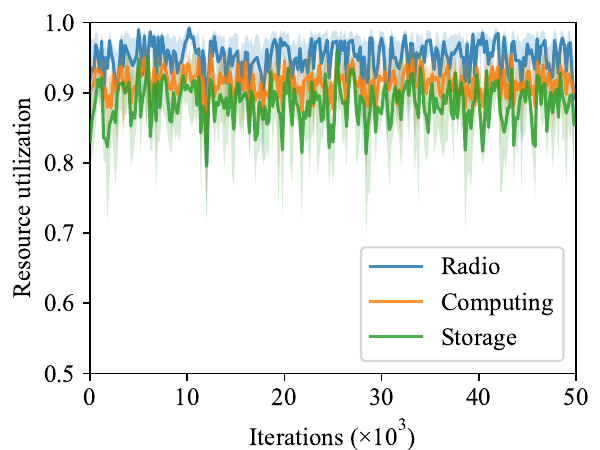}%
}
\subfloat[ILP]{\includegraphics[width=0.248\textwidth]{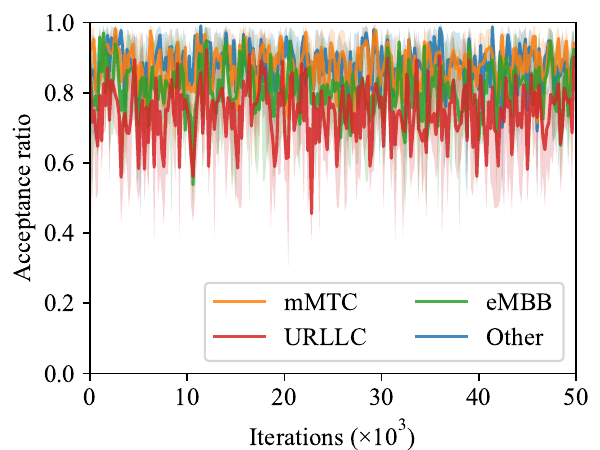}%
}

\subfloat[Greedy]{\includegraphics[width=0.248\textwidth]{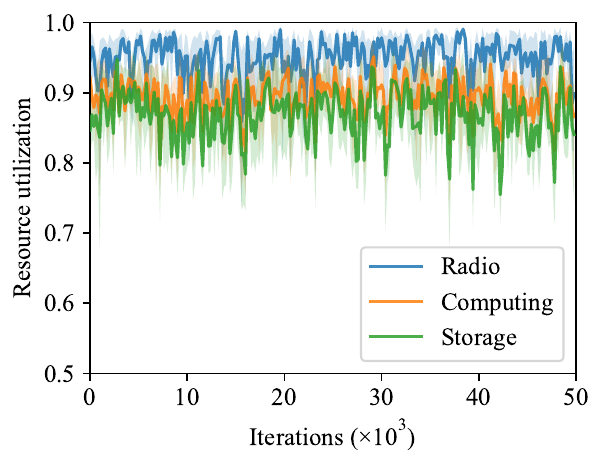}%
}
\subfloat[Greedy]{\includegraphics[width=0.248\textwidth]{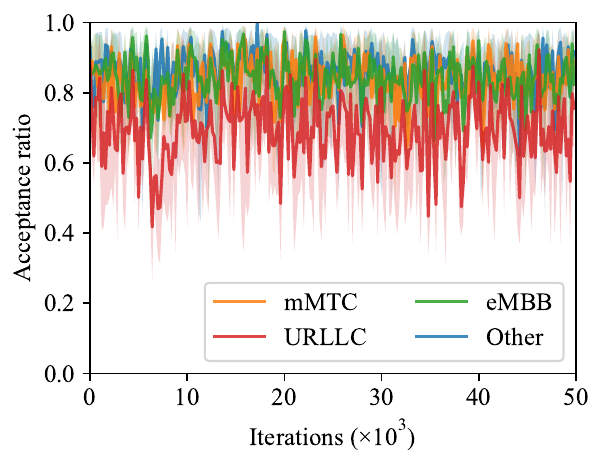}%
}
\caption{resource utilization and acceptance ratio in default admission policies}
\label{EnvDe}
\end{figure}

The number of training epochs for step 2 and step 3 in \textbf{Algorithm \ref{al2}} is set as 6000 and 50000 respectively. Step 2, which exclusively focuses on training the critic network, does not interfere with the network operation when the default policy is accurately replicated. Therefore, we focus on analyzing the performance in step 3. Firstly, we adopt three default policies for 50,000 decision epochs, with the resource utilization and acceptance ratio illustrated in Fig. \ref{EnvDe}. Because the stochastic arrival and service process will hinder the performance comparison of different policies, we record data every 200 epochs and conduct four experiments using different random seeds. The solid lines in the figures represent the average values across multiple experiments, while the shaded areas denote the 75\% error bar. The curves highlight the characteristics of different policies. In the PRIO policy, the acceptance ratio of services in different slices follows the pre-defined priority order, as shown in Fig. \ref{PRIOAR}. In contrast, the ILP and Greedy policies achieve relatively higher radio resource utilization by accepting more eMBB and `other' requests.

\begin{figure}[!t]
\centering
\subfloat[Dueling-DQN]{\includegraphics[width=0.248\textwidth]{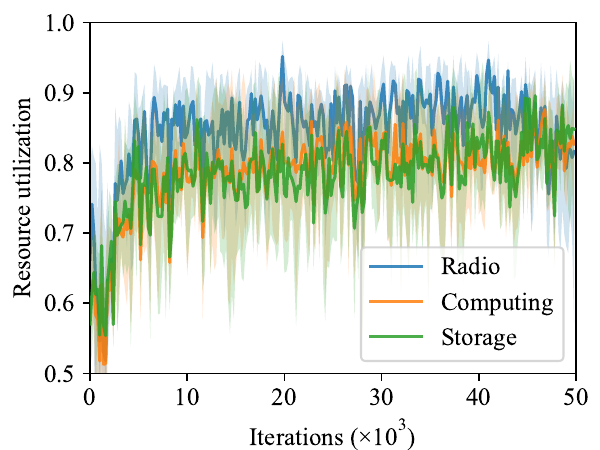}%
\label{dqn1}}
\subfloat[Dueling-DQN]{\includegraphics[width=0.248\textwidth]{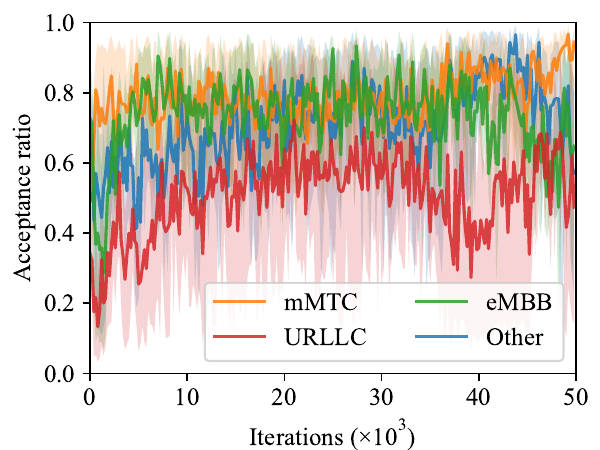}%
\label{dqn2}}

\subfloat[Modified A2C]{\includegraphics[width=0.248\textwidth]{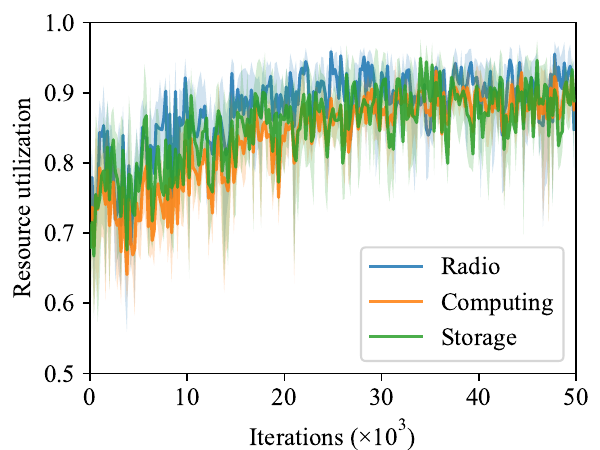}%
\label{a}}
\subfloat[Modified A2C]{\includegraphics[width=0.248\textwidth]{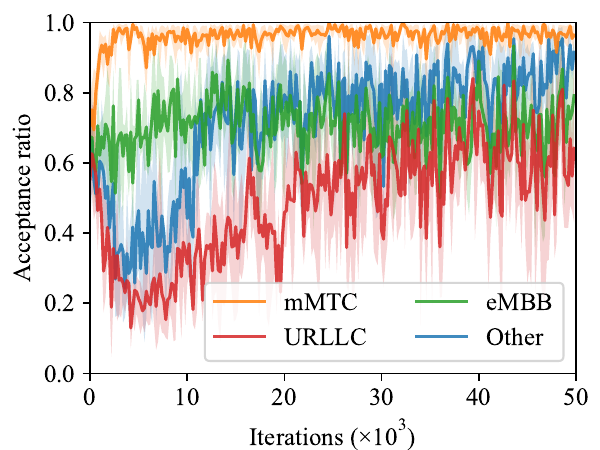}%
\label{b}}

\caption{resource utilization and acceptance ratio in directly trained DRL}
\label{EnvRL}
\end{figure}

We subsequently conduct direct reinforcement learning for two DRL models: the state-of-the-art Dueling-DQN, and our proposed DRL model, the modified A2C, within the network environment, as illustrated in Figure \ref{EnvRL}. During the initial training phases, both directly trained models exhibited stochastic behavior, resulting in comparatively low resource utilization and an unstable acceptance ratio. Furthermore, the Dueling-DQN, lacking a customized output layer for handling queued requests, encountered challenges in achieving convergence and maintaining stability, as indicated by the wider shaded areas. After approximately 20,000 decision epochs, as our DRL model gradually converges, we observe a plateau in resource utilization as well as the stabilization of the acceptance ratio. According to the acceptance ratio curves, modified A2C exhibits a tendency to accept more mMTC and Other requests to increase storage resource occupation.

Next, we implement the DT-assisted DRL solution based on different default policies. In contrast to directly trained models, all DT-assisted DRL models maintain high resource utilization throughout the entire training phase. At the beginning of training, the acceptance ratio pattern in DT-assisted DRL shows consistency with that in default policy, as illustrated on the left side of Figures \ref{d} and \ref{PRIOAR}. When the training progresses, the acceptance ratio gradually evolves and eventually aligns with that in the directly trained DRL, as depicted on the right side of Figures \ref{d} and \ref{b}.

\begin{figure}[!t]
\centering
\subfloat[DRL with DT (PRIO)]{\includegraphics[width=0.248\textwidth]{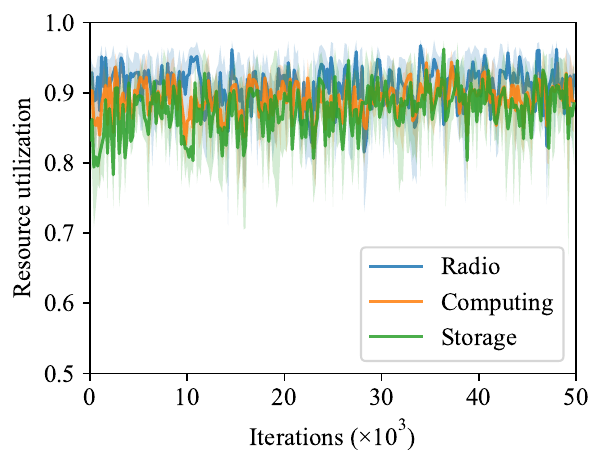}%
}
\subfloat[DRL with DT (PRIO)]{\includegraphics[width=0.248\textwidth]{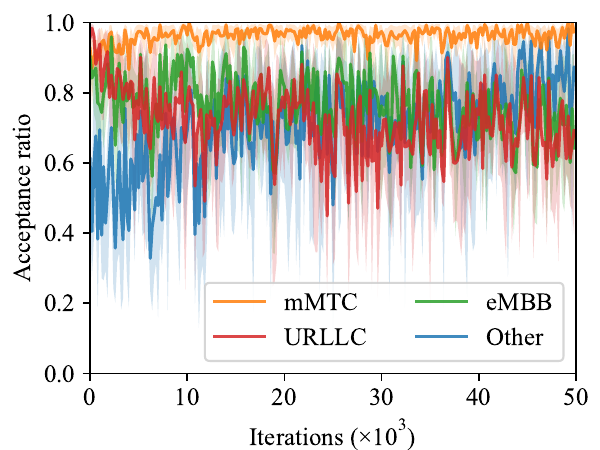}%
\label{d}}

\subfloat[DRL with DT (ILP)]{\includegraphics[width=0.248\textwidth]{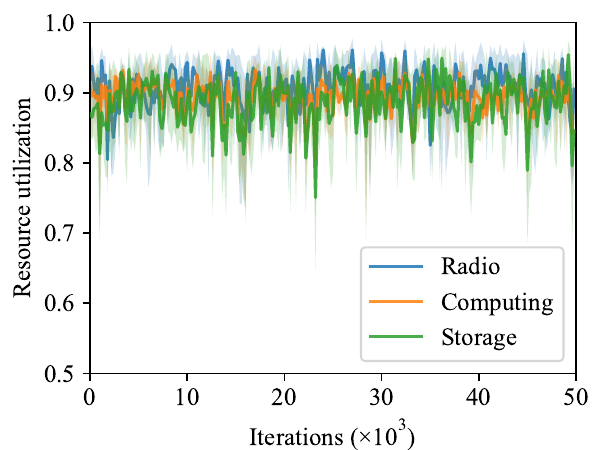}%
}
\subfloat[DRL with DT (ILP)]{\includegraphics[width=0.248\textwidth]{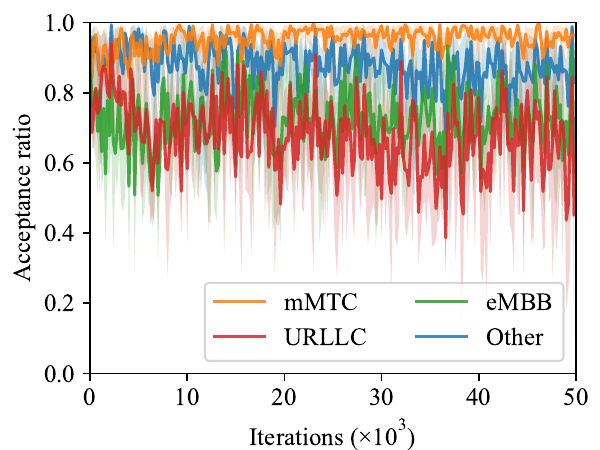}%
}

\subfloat[DRL with DT (Greedy)]{\includegraphics[width=0.248\textwidth]{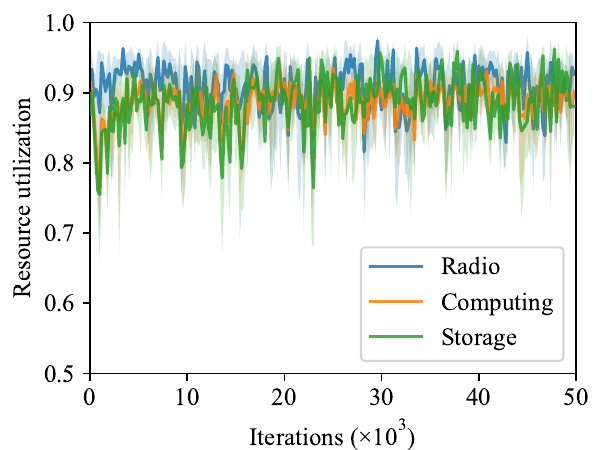}%
}
\subfloat[DRL with DT (Greedy)]{\includegraphics[width=0.248\textwidth]{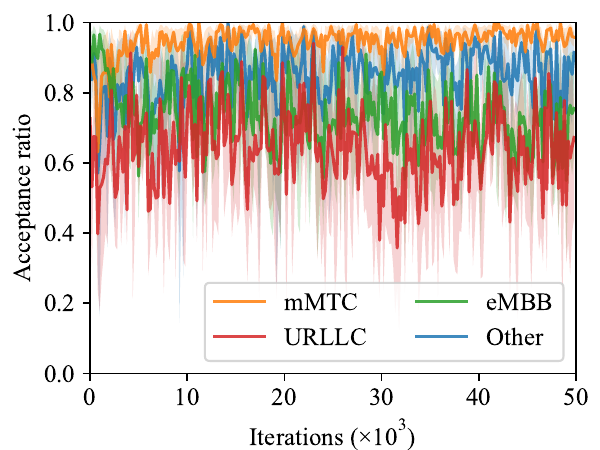}}%

\caption{resource utilization and acceptance ratio in DT-assisted DRL}
\end{figure}

\begin{figure*}[!t]
\centering
\includegraphics[width=1\textwidth]{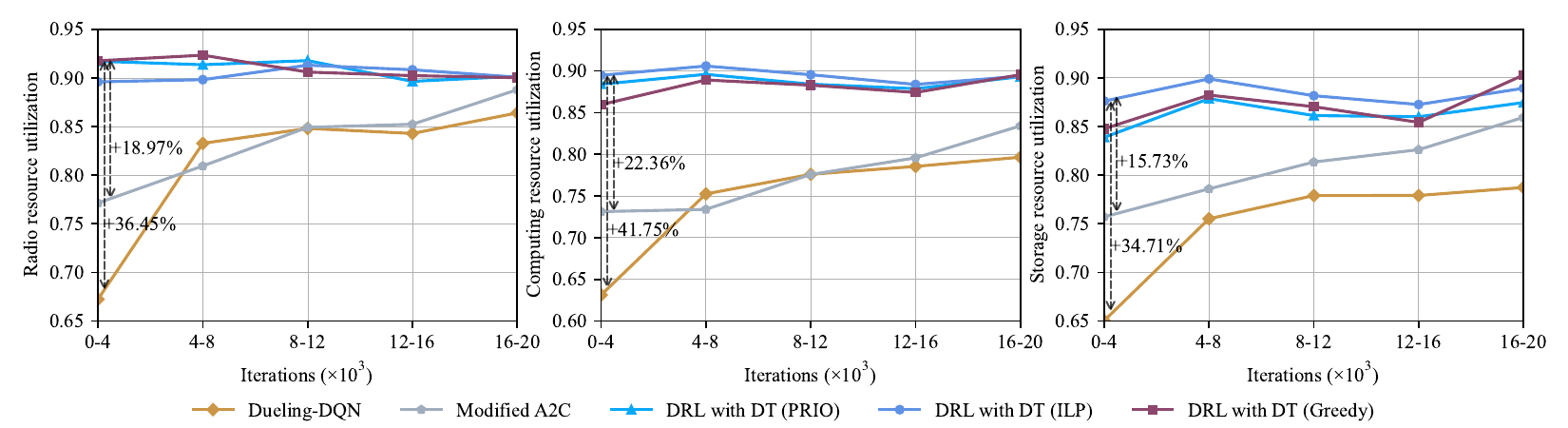}
\caption{Comparison of resource utilization in different methods during the early training stage}
\label{res_uti}
\end{figure*}

To quantitatively analyze resource utilization performance between directly trained DRL and DT-assisted DRL methods, we evaluate results from the first 20,000 decision epochs, aggregate data in 4,000-epoch intervals, and present line charts for each resource type. As depicted in Fig. \ref{res_uti}, all three DT-assisted DRL methods demonstrate a notable advantage in resource utilization over the directly trained DRL method. Specifically, within the first 4,000 epochs, DT-assisted DRL outperforms the state-of-the-art Dueling-DQN by a substantial margin, with resource utilization improvements up to 41.75\%. Moreover, to eliminate the influence of model differences, we also assess the performance of DT-assisted DRL against the directly trained identical model. The results show that the DT assistance yields an exclusive enhancement in resource utilization of up to 22.36\%. These disparities in resource utilization tend to diminish as the models converge gradually.

\begin{figure}[!t]
\centering
\includegraphics[width=0.5\textwidth]{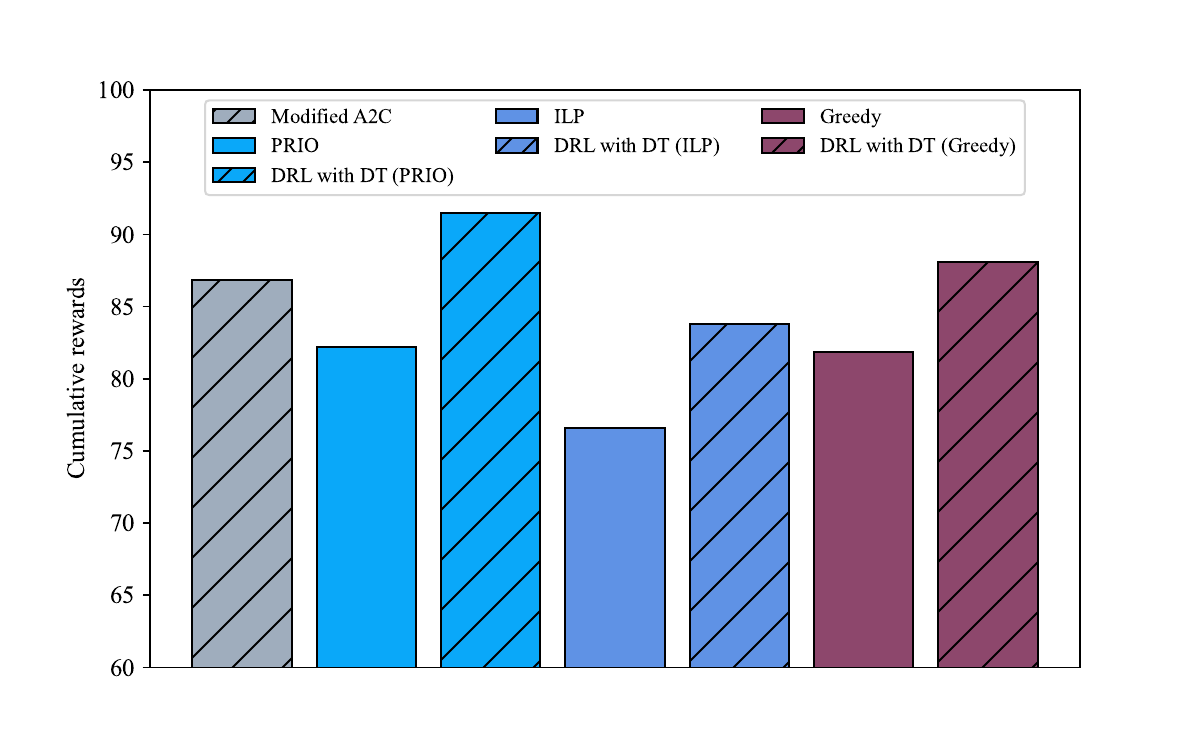}
\caption{Average reward in 400 iterations after training}
\label{curewards}
\end{figure}

Furthermore, we compare the cumulative rewards using different methods to examine whether the optimization goal has been achieved. Fig. \ref{curewards} illustrates the cumulative rewards (total storage-based revenue) over 400 decision epochs, where all DRL models outperform the default admission policies in the preset target. Additionally, we observe that the default admission policy can influence the performance of DT-assisted DRL to a certain extent. When default policies achieved relatively high storage revenues (PRIO and Greedy), DT-assisted DRL performed similarly or better than directly trained DRL. In contrast, the ILP policy's deficiency in storage revenue leads the ILP-based DT-assisted DRL to underperform compared to directly trained DRL. Nevertheless, this phenomenon primarily stems from the limited number of training samples, and we suppose that it will diminish as the models converge further after a substantial number of decision epochs.
 



\section{Conclusion}
In this paper, we have investigated the instability of conventional DRL methods for admission control in a sliced wireless network with request queues and combinatorial radio, computing, and storage resources. We have formulated the admission decision-making process as a semi-Markov decision process and subsequently simplified it into an equivalent discrete-time Markov decision process. To deal with the stochasticity of DRL, we have constructed a DT network of admission policy using supervised learning and proposed a DT-assisted online DRL solution. Extensive simulations demonstrated that the DT-assisted DRL model increased resource utilization by over 40\% compared to directly trained state-of-the-art Dueling-DQN and over 20\% compared to the directly proposed model during the initial training stage. 
This performance improvement is achieved while retaining the ability to optimize long-term rewards, thereby mitigating the risk of deploying DRL in practical wireless networks while sustaining its effectiveness. Meanwhile, the robust performance using a straightforward greedy policy implies that in case the default admission policy is too complex to replicate, like policies incorporating request prediction, a simple substitute policy could still be utilized to implement the proposed solution.





{\appendices
\makeatletter  
\section*{Appendix A\\Proof of Theorem 1}
\begin{proof}
An embedded Markov chain without disjoint closed sets implies the system definitely revisits a particular state after a certain number of events, thus exhibiting the properties of a renewal process. Fix the initial state $\mathbf{s}_0$ and define the cycle as the time between two successive transitions into state $\mathbf{s}_0$. According to the renewal reward theorem for renewal reward processes, we have
\begin{equation}\label{eqrt}
    \lim _{t \rightarrow \infty} \frac{R(t)}{t}=\frac{\mathbb{E}\left[R_1\right]}{\mathbb{E}\left[T_1\right]}, 
\end{equation}
where $R_1$ represents the total rewards earned in the first renewal cycle, $T_1$ represents the length of the first renewal cycle, and $\mathbb{E}[.]$ denotes the expectation. Also, by the expected-value version of the renewal-reward theorem, it follows
\begin{align}
     \lim _{m \rightarrow \infty} \frac{\mathbb{E}\left[\sum _{i=1}^{m}r_i\right]}{m}=\frac{\mathbb{E}\left[R_1\right]}{\mathbb{E}\left[N_1\right]},\\
    \lim _{m \rightarrow \infty} \frac{\mathbb{E}\left[\sum _{i=1}^{m}\tau_i\right]}{m}=\frac{\mathbb{E}\left[T_1\right]}{\mathbb{E}\left[N_1\right]},   \label{eqet}    
\end{align}
where $r_i$ and $\tau_i$ denote the reward and the sojourn time over the $i$-th epoch respectively, and $N_1$ represents the number of epochs in the first renewal cycle. From (\ref{eqrt})-(\ref{eqet}), we have
\begin{equation}\label{eqrtt}
    \lim _{t \rightarrow \infty} \frac{R(t)}{t}=\lim _{m \rightarrow \infty} \frac{\mathbb{E}\left[\sum _{i=1}^{m}r_i\right]}{\mathbb{E}\left[\sum _{i=1}^{m}\tau_i\right]}.
\end{equation}
Due to the Markovian property of the reward and sojourn time, we have
\begin{align}
    \mathbb{E}\left[\sum _{i=1}^{m}r_i\right] = \sum_{i=1}^{m} \sum_{\mathbf{s}\in \mathcal{S}} r(\mathbf{s},\mathbf{a}_{\mathbf{s}|\pi}) p^{(i)}_{\mathbf{s},\mathbf{a}^{}_{\mathbf{s}|\pi},\mathbf{s}^\prime},\label{eqeri}\\
    \mathbb{E}\left[\sum _{i=1}^{m}\tau_i\right] = \sum_{i=1}^{m} \sum_{\mathbf{s}\in \mathcal{S}} \tau(\mathbf{s},\mathbf{a}_{\mathbf{s}|\pi}) p^{(i)}_{\mathbf{s},\mathbf{a}_{\mathbf{s}|\pi},\mathbf{s}^\prime},\label{eqeti}
\end{align}
where $p^{(i)}_{\mathbf{s},\mathbf{a}_{\mathbf{s}|\pi},\mathbf{s}^\prime}= \text{Pr}(S_{t+i}=\mathbf{s}^\prime | S_{t}=\mathbf{s},A_t =\mathbf{a}_{\mathbf{s}|\pi})$ refers to the $i$-step transition probability under policy $\pi$. By leveraging the relationship between $i$-step transition probability and equilibrium probability
\begin{equation}
    \lim _{m \rightarrow \infty} \frac{1}{m}\sum_{i=1}^{m} p^{(i)}_{\mathbf{s},\mathbf{a}_{\mathbf{s}|\pi},\mathbf{s}^\prime} = \omega(\mathbf{s}|\pi)
\end{equation}
and substituting into (\ref{eqrtt}) with (\ref{eqeri}) and (\ref{eqeti}), we obtain
\begin{equation}
    \lim _{t \rightarrow \infty} \frac{R(t)}{t}= \frac{\sum_{\mathbf{s} \in \mathcal{S}} r\left(\mathbf{s},\mathbf{a}_{\mathbf{s}|\pi}\right) \omega(\mathbf{s}|\pi) }{ \sum_{\mathbf{s} \in \mathcal{S}} \tau\left(\mathbf{s},\mathbf{a}_{\mathbf{s}|\pi}\right) \omega(\mathbf{s}|\pi)}.
\end{equation}
\end{proof}

\section*{Appendix B\\Proof of Theorem 2}
\begin{proof}
The equilibrium probabilities $\bar{\omega}(\mathbf{s}|\pi)$ in discrete-time MDP satisfy the following equilibrium equation
\begin{equation}\label{wbarspi}
        \bar{\omega}(\mathbf{s}|\pi) = \sum_{\mathbf{s}_0\in \mathcal{S}} \bar{\omega}(\mathbf{s}_0|\pi) \Bar{p}_{\mathbf{s_0},\mathbf{a}_{\mathbf{s}_0|\pi},\mathbf{s}}
\end{equation}
By substituting $\Bar{p}_{\mathbf{s_0},\mathbf{a}_{\mathbf{s}_0|\pi},\mathbf{s}}$ with (\ref{transitionp}), we obtain

\begin{equation}
\label{wspi}
\begin{split}
    \bar{\omega}(\mathbf{s}|\pi) =\sum_{\mathbf{s}_0\in \mathcal{S}} \bar{\omega}(\mathbf{s}_0|\pi) \frac{\tau_0}{\tau(\mathbf{s}_0,\pi_{\mathbf{s}_0})}&p_{\mathbf{s_0},\mathbf{a}_{\mathbf{s}_0|\pi},\mathbf{s}}+\\
    &\bar{\omega}(\mathbf{s}|\pi)\left(1- \frac{\tau_0}{\tau(\mathbf{s},\pi_{\mathbf{s}})}\right).
\end{split}    
\end{equation}
By eliminating $\bar{\omega}(\mathbf{s}|\pi)$ on both sides of this equation and dividing by $\tau_0$, (\ref{wspi}) is rewritten as
\begin{equation}
    \frac{\bar{\omega}(\mathbf{s}|\pi)}{\tau(\mathbf{s},\pi_{\mathbf{s}})} = \sum_{\mathbf{s}_0\in \mathcal{S}} \frac{\bar{\omega}(\mathbf{s}_0|\pi)}{\tau(\mathbf{s}_0,\pi_{\mathbf{s}_0})} p_{\mathbf{s_0},\mathbf{a}_{\mathbf{s}_0|\pi},\mathbf{s}}.
\end{equation}
Notice that the embedded Markov chain in SMDP also satisfies the equilibrium equation in (\ref{equili}). Thus, for a certain constant $\gamma>0$, we have
\begin{equation}
    \omega(\mathbf{s}|\pi) = \gamma\frac{\bar{\omega}(\mathbf{s}|\pi)}{\tau(\mathbf{s},\pi_{\mathbf{s}})}.
    \label{equi}
\end{equation}
Since $\sum_{\mathbf{s}\in \mathcal{S}} \bar{\omega}_{\mathbf{s}} = 1$, we can choose the value of the constant as $\gamma = \sum_{\mathbf{s}\in \mathcal{S}} \tau(\mathbf{s},\pi_{\mathbf{s}}) \omega(\mathbf{s}|\pi)$. Finally, using (\ref{averagerevenue}), (\ref{rewardbar}), and (\ref{equi}), the long-term average reward of the equivalent discrete-time MDP is derived as follows

\begin{align}
    \bar{g}(\pi) &= \sum_{\mathbf{s}\in \mathcal{S}} \bar{r}_\mathbf{s}\left(\mathbf{a}_{\mathbf{s}|\pi}\right)\bar{\omega}(\mathbf{s}|\pi)\notag\\
    &= \sum_{\mathbf{s}\in \mathcal{S}} \frac{r\left(\mathbf{s},\mathbf{a}_{\mathbf{s}|\pi}\right)}{\tau\left(\mathbf{s},\mathbf{a}_{\mathbf{s}|\pi}\right)} \frac{\omega(\mathbf{s}|\pi) \tau(\mathbf{s},\pi_{\mathbf{s}})}{\sum_{\mathbf{s}\in \mathcal{S}} \tau(\mathbf{s},\pi_{\mathbf{s}}) \omega(\mathbf{s}|\pi)}\notag\\
    & = \frac{\sum_{\mathbf{s} \in \mathcal{S}} r\left(\mathbf{s},\mathbf{a}_{\mathbf{s}|\pi}\right) \omega(\mathbf{s}|\pi) }{ \sum_{\mathbf{s} \in \mathcal{S}} \tau\left(\mathbf{s},\mathbf{a}_{\mathbf{s}|\pi}\right) \omega(\mathbf{s}|\pi)} = g(\pi)
\end{align}

\end{proof}
\ifCLASSOPTIONcaptionsoff
  \newpage
\fi



\footnotesize
\bibliographystyle{IEEEtran}

\bibliography{IEEEabrv,IEEEexample}
\end{document}